\documentclass[10pt,journal,compsoc]{IEEEtran}

\usepackage{commons}
\usepackage{macros}

\usepackage[T1]{fontenc}    
\usepackage{hyperref}       
\usepackage{url}            
\usepackage{microtype}      

\usepackage{array}
\usepackage{float}

\usepackage{caption}
\usepackage{subcaption}

\usepackage{tikz}
\usetikzlibrary{bayesnet}

\newcommand{\N}{\mathcal{N}}
\renewcommand{\L}{\mathcal{L}}
\newcommand{\D}{\mathcal{D}}
 
\newcommand{\w}{\mathbf{w}}
\newcommand{\x}{\mathbf{x}}
\newcommand{\y}{\mathbf{y}}
\newcommand{\z}{\mathbf{z}}

\DeclareMathOperator{\TC}{TC}

\title{ Information Dropout: Learning Optimal Representations Through Noisy Computation\thanks{Dedicated to Naftali Tishby in the occasion of the conference \emph{Information, Control and Learning} held in his honor in Jerusalem, September 26-28, 2016. Registered as Tech Report UCLA-CSD160009 and arXiv:1611.01353 on November 6, 2016} }

\author{\IEEEauthorblockN{Alessandro Achille and Stefano Soatto}\\
\IEEEauthorblockA{Department of Computer Science\\
University of California, Los Angeles\\
405 Hilgard Ave, Los Angeles, 90095, CA, USA\\
Email: \{achille,soatto\}@cs.ucla.edu}}

\begin{document}

\maketitle

\begin{abstract}

The cross-entropy loss commonly used in deep learning is closely related to the defining properties of optimal representations, but does not enforce some of the key properties.
We show that this can be solved by adding a regularization term, which is in turn related to injecting multiplicative noise in the activations of a Deep Neural Network, a special case of which is the common practice of dropout. We show that our regularized loss function can be efficiently minimized using Information Dropout, a generalization of dropout rooted in information theoretic principles that automatically adapts to the data and can better exploit architectures of limited capacity.
When the task is the reconstruction of the input, we show that 
our loss function yields a Variational Autoencoder as a special case,
thus providing a link between representation learning, information theory and variational inference. Finally, we prove that we can promote the creation of disentangled representations simply by enforcing a factorized prior,
a fact that has been observed empirically in recent work.
Our experiments validate the theoretical intuitions behind our method,
and we find that information dropout
achieves a comparable or better generalization performance than binary dropout,
especially on smaller
models, since it can automatically adapt the noise to the structure of the network, as well as to the test sample.

\end{abstract}

\section{Introduction}

We call ``representation'' any function of the data that is useful for a task. An optimal representation is most useful (sufficient),  parsimonious (minimal), and minimally affected by nuisance factors (invariant). Do deep neural networks approximate such {\em sufficient invariants}? 

The cross-entropy loss most commonly used in deep learning does indeed enforce the creation of sufficient representations, but the other defining properties of optimal representations do not seem to be explicitly enforced by the commonly used training procedures.
However, we show that this can
be done by adding a regularizer, which is related to the injection of multiplicative noise in the activations, with the surprising result that \emph{noisy computation facilitates the approximation of optimal representations}.
In this paper we establish connections between the theory of optimal representations for classification tasks, variational inference, dropout and disentangling in deep neural networks. Our contributions can be summarized in the following steps:

\begin{enumerate}[(1)]

\item We define optimal representations using established principles of statistical decision and information theory: sufficiency, minimality, invariance (cf. \cite{soatto2016visual,tishby2000information}) (\Cref{sec:previous-work}).

\item We relate the defining properties of optimal representations for classification to the loss function most commonly used in deep learning, but with an added regularizer (\Cref{sect-optimal}, \cref{eq:empirical-IB}).

\item We show that, counter-intuitively, injecting multiplicative noise to the computation improves the properties of a representation and results in better approximation of an optimal one (\Cref{sect-information}).

\item We relate such a multiplicative noise to the regularizer, and show that in the special case of Bernoulli noise, regularization reduces to dropout \cite{srivastava2014dropout}, thus establishing a connection to information theoretic principles. We also provide a more efficient alternative, called Information Dropout, that makes better use of limited capacity, adapts to the data, and is related to Variational Dropout \cite{kingma2015variational} (\Cref{sect-information}).

\item We show that, when the task is reconstruction, the procedure above yields a generalization of the Variational Autoencoder, which is instead derived from a Bayesian inference perspective \cite{kingma2013auto}. This establishes a connection between information theoretic and Bayesian representations, where the former explains the use of a multiplier used in practice but unexplained by Bayesian theory (\Cref{sec:vae}).

\item We show that ``disentanglement of the hidden causes,'' an often-cited but seldom formalized desideratum for deep networks, can be achieved by
assuming a factorized prior
for the components of the optimal representation. Specifically, we prove that computing the
regularizer term under the simplifying assumption of an independent prior has the effect of  minimizing the total correlation of the components,
a phenomenon previously observed empirically by \cite{higgins17beta} (\Cref{sec:disentanglement}).

\item We validate the theory with several experiments including: improved insensitivity/invariance to nuisance factors using Information Dropout using (a) Cluttered MNIST \cite{mnih2014recurrent} and (b) MNIST+CIFAR, a newly introduced dataset to test sensitivity to occlusion phenomena critical in Vision applications; (c) we show improved efficiency of Information Dropout compared to regular dropout for limited capacity networks, (d) we show that Information Dropout favors disentangled representations; (e) we show that Information Dropout adapts to the data and allows different amounts of information to flow between different layers in a deep network (\Cref{sect-experiments}).

\end{enumerate}
In the next section we introduce the basic formalism to make the above statements more precise, which we do in subsequent sections.

\section{Preliminaries}
\label{sect-preliminaries}

In the general supervised setting,  we want to learn
the conditional distribution $p(\y|\x)$ of some random variable $\y$,
which we refer to as the \emph{task},
given (samples of the) input data $\x$.
In typical applications, $\x$ is often high dimensional
(for example an image or a video), while $\y$ is low dimensional,
such as a label or a coarsely-quantized location.
In such cases, a large part of the variability in $\x$
is actually due to \emph{nuisance factors} that affect
the data, but are otherwise irrelevant for the task \cite{soatto2016visual}.
Since by definition these nuisance factors are not predictive of the task,
they should be disregarded during the inference process.
However, it often happens that modern machine learning algorithms,
in part due to their high flexibility, 
will fit spurious correlations, present in the training data,
between the nuisances and the task, 
thus leading to poor generalization performance.

In view of this, \cite{tishby2015deep}  argue that the success of deep learning 
is in part due to the capability of neural networks to build
incrementally better representations that expose
the relevant variability, while at the same time discarding nuisances.
This interpretation is intriguing, as it establishes a connection
between machine learning, probabilistic inference, and information theory.
However, common training practice does not seem to stem from this insight, and indeed deep networks may maintain even in the top layers dependencies on easily ignorable nuisances (see for example \Cref{fig:nuisances}).

To bring the practice in line with the theory, and to better understand these connections,
we introduce a modified cost function, that can be seen as an approximation of
the Information Bottleneck Lagrangian of \cite{tishby2000information}, which encourages
the creation of representations of the data which are increasingly disentangled and
insensitive to the action of nuisances,
and we show that this loss can be minimized using a new layer, which we call
\emph{Information Dropout},
that allows the network to selectively introduce multiplicative noise in the layer activations,
and thus to control the flow of information.
As we show in various experiments, this method improves the generalization performance
by building better representations and preventing overfitting,
and it considerably improves over binary dropout on smaller models, since, unlike dropout, Information Dropout also adapts the noise to the structure of the network and to the individual sample at test time.

Apart from the practical interest of Information Dropout, one of our main results
is that Information Dropout can be seen as a generalization to
several existing dropout methods, providing a unified framework to analyze them,
together with some additional insights on empirical results.
As we discuss in \Cref{sec:previous-work},
the introduction of noise to prevent overfitting has already been studied from
several points of view. For example the original formulation of dropout of \cite{srivastava2014dropout},
which introduces binary multiplicative noise, was motivated as a way of efficiently
training an ensemble of exponentially many networks, that would be averaged at testing time.
\cite{kingma2015variational} introduce \emph{Variational Dropout}, 
a dropout method which closely resemble ours,
and is instead derived from a Bayesian analysis
of neural networks.
Information Dropout gives an alternative information-theoretic interpretation of those methods.

As we show in \Cref{sec:vae}, other than being very closely related to Variational Dropout,
Information Dropout directly yields a variational autoencoder as a special case
when the task is the reconstruction of the input.
This result is in part expected, since our loss function seeks an optimal
representation of the input for the task of reconstruction, and the representation given by the latent
variables of a variational autoencoder fits the criteria.
However, it still rises the question of exactly what and how deep are the links between
information theory, representation learning, variational inference and nuisance invariance.
This work can be seen as a small step in answering this question.

\section{Related work}
\label{sec:previous-work}

The main contribution of our work is to establish how two seemingly
different areas of research, namely dropout methods
to prevent overfitting, and the study of optimal representations,
can be linked through the Information Bottleneck principle.

Dropout was introduced by Srivastava et al. \cite{srivastava2014dropout}.
The original motivation was that by randomly dropping the activations during training,
we can effectively train an ensemble of exponentially many networks,
that are then averaged during testing, therefore reducing overfitting.
Wang~et~al.~\cite{wang2013fast} suggested that dropout could be seen as performing
a Monte-Carlo approximation of an implicit loss function, and that instead
of multiplying the activations by binary noise, like in the original dropout,
multiplicative Gaussian noise with mean 1 can be used
as a way of better approximating the implicit loss function.
This led to a comparable performance but faster training than binary dropout.

Kingma~et~al.~\cite{kingma2015variational} take a similar view of dropout as introducing
multiplicative (Gaussian) noise, but instead study the problem from a Bayesian point of view.
In this setting, given a training dataset $\D=\set{(\x_i,\y_i)}_{i=1,\ldots,N}$ and a prior distribution $p(\w)$, we want to compute the posterior
distribution $p(\w|\D)$ of the weights $\w$ of the network.
As is customary in variational inference, the true posterior can be approximated by minimizing the negative variational lower bound
$\L(\theta)$ of the marginal log-likelihood of the data,
{
\begin{multline}
    \L(\theta) = \frac{1}{N} \sum_{i=1}^N \E_{\w \sim p_\theta(\w|\D)}[-\log p(\y_i|\x_i,\w)] + \\ + \frac{1}{N}\KL{p_\theta(\w|\D)}{p(\w)}.
\end{multline}
\label{eq:var-dropout}
}
This minimization is difficult to perform, since it requires to repeatedly sample new weights for
each sample of the dataset. As an alternative, \cite{kingma2015variational} suggest
that the uncertainty about the weights that is expressed by the posterior 
distribution $p_\theta(\w|\D)$ can equivalently be encoded as a multiplicative noise
in the activations of the layers (the so called \emph{local reparametrization trick}).
As we will see in the following sections, this loss function closely resemble the one 
of Information Dropout, which however is derived from a purely 
information theoretic argument based on the Information Bottleneck principle.
One difference is that we allow the parameters of the noise to change on a per-sample basis 
(which, as we show in the experiments, can be useful to deal with nuisances),
and that we allow a scaling constant $\beta$ in front of the KL-divergence term,
which can be changed freely. 
Interestingly, even if the Bayesian derivation does not allow a rescaling of the KL-divergence,
\cite{kingma2015variational} notice that choosing a different scale for the KL-divergence term can indeed lead to 
improvements in practice. A related method, but derived from an information theoretic perspective was also suggested previously by \cite{hinton1993keeping}.

The interpretation of deep neural network as a way of creating successively better representations of the data has 
already been suggested and explored by many.
Most recently, Tishby et al. \cite{tishby2015deep} put forth
an interpretation of deep neural networks as creating sufficient representations of the data that 
are increasingly minimal.
In parallel simultaneous work, \cite{alemi2016deep} 
approximate the information bottleneck similarly to us, but focus on empirical analysis of robustness to adversarial perturbations rather than tackling disentanglement, invariance and minimality analytically.
Some have focused on creating representations that are \emph{maximally} invariant to nuisances, especially when they have the structure of a (possibly infinite-dimensional) group acting on the data,
like \cite{sundaramoorthiPVS09}, 
or, when the nuisance is a locally compact group acting on each layer, by successive approximations implemented by 
hierarchical convolutional architectures, like \cite{anselmi2016invariance} and \cite{bruna2011classification}.
In these cases, which cover common nuisances such as translations and
rotations of an image (affine group), or small diffeomorphic deformations due to a slight change of
point of view (group of diffeomorphisms), the representation is equivalent to the data modulo the action
of the group.
However, when the nuisances are not a group, as is the case for occlusions, it is not possible to achieve such equivalence, that is, there is a loss. To address this problem, \cite{soatto2016visual} defined optimal representations not in terms of maximality, but in terms of \emph{sufficiency}, and characterized representations that are both sufficient and invariant. They argue that the management of nuisance factors common in visual data, such as changes of viewpoint, local deformations, and changes of illumination, is directly tied to the specific structure of deep convolutional networks, where local marginalization of simple nuisances at each layer results
in marginalization of complex nuisances in the network as a whole.

Our work fits in this last line of thinking, where the goal is not equivalence to the data up to the action of (group) nuisances, but instead sufficiency for the task. \textbf{Our main contribution} in this sense
is to show that injecting noise into the layers, and therefore
using a non-deterministic function of the data, can actually simplify the theoretical analysis and
lead to disentangling and improved insensitivity to nuisances. This is an alternate explanation to that put forth by the references above.

\section{Optimal representations and the Information Bottleneck loss}
\label{sect-optimal}

Given some input data $\x$, we want to compute
some (possibly nondeterministic) function of $\x$, called a \emph{representation},
that has some desirable properties in view of the task $\y$, for instance by being more convenient to work with,
exposing relevant statistics, or being easier to store.
Ideally, we want this representation to be as good as the original data for the task, and not squander resources modeling parts of the data that are irrelevant to the task.
Formally, this means that we want to find a random variable $\z$ satisfying the following conditions:
\begin{enumerate}[(i)]
    \item $\z$ is a \textbf{representation} of $\x$; that is, its distribution depends only on $\x$, as expressed by the following Markov chain:
        \begin{center}
        \begin{tikzpicture}[]
            \node[latent] (y) {$y$};
            \node[obs, right = of y] (x) {$x$};
            \node[latent, right = of x] (z) {$z$};
            
            \edge[] {y}{x};
            \edge[] {x}{z};
        \end{tikzpicture}
        \end{center}
    \item $\z$ is \textbf{sufficient} for the task $\y$,
    that is $I(\x;\y) = I(\z;\y)$, expressed by the Markov chain:
        \begin{center}
        \begin{tikzpicture}[]
            \node[latent] (y) {$y$};
            \node[latent, right = of y] (z) {$z$};
            \node[obs, right = of z] (x) {$x$};
            
            \edge[] {y}{z};
            \edge[] {z}{x};
        \end{tikzpicture}
        \end{center}
    \item among all random variables satisfying these requirements, the mutual information $I(\x;\z)$
    is \textbf{minimal}. This means that $\z$ discards all variability in the data that is not relevant to the task.
\end{enumerate}
Using the identity $I(\x;\y) - I(\z;\y) = H(\y|\z) - H(\y|\x)$, where $H$ denotes the entropy and $I$ the mutual information, it is easy to see that 
the above conditions are equivalent
to finding a distribution $p(\z|\x)$ which solves the optimization problem
\begin{align*} 
\mathrm{minimize}\quad & I(\x;\z) \\
\text{s.t.}\quad  & H(\y|\z) = H(\y|\x).
\end{align*}
The minimization above is difficult in general. For this reason, Tishby et al. have introduced a generalization known as the \emph{Information Bottleneck Principle} and the associated Lagrangian to be minimized \cite{tishby2000information}
\begin{equation}
\label{eq:IB-lagrangian}
\L = H(\y|\z) + \beta I(\x ; \z).
\end{equation}
where $\beta$ is a positive constant that manages the trade-off between sufficiency (the performance on the task, as measured by the first term) and minimality (the complexity of the representation, measured by the second term). It is easy to see that, in the limit $\beta \to 0^+$, this  is equivalent to the original problem, where $\z$ is a minimal sufficient statistic. When all random variables are discrete
and $\z=T(\x)$ is a deterministic function of $\x$,
the algorithm proposed by \cite{tishby2000information} can be used to minimize
the IB Lagrangian efficiently.
However, no algorithm is known to minimize
the IB Lagrangian for non-Gaussian, high-dimensional continuous random variables.

One of our key results is that, when we restrict to the family of distributions
obtained by injecting noise to one layer of a neural network, we can efficiently approximate
and minimize the IB Lagrangian.\footnote{Since we restrict the family of distributions, there is no guarantee that the resulting representation will be optimal. We can, however, iterate the process to obtain incrementally improved approximations.
} 
As we will show, this process can be effectively implemented  through 
a generalization of the dropout layer that we call \emph{Information Dropout}.

To set the stage, we rewrite the IB Lagrangian as a per-sample loss function. Let $p(\x,\y)$ denote the true distribution of the data, from which the training set $\set{(\x_i, \y_i)}_{i=1,\ldots,N}$ is sampled, and let  $p_\theta(\z|\x)$ and $p_\theta(\y|\z)$ denote the unknown distributions that we wish to estimate, parametrized by $\theta$. Then, we can write the two terms in the IB Lagrangian as 
\begin{align*}
H(\y|\z) &\simeq \E_{\x,\y \sim p(\x,\y)}\bra{\E_{\z \sim p_\theta(\z|\x)}  [-\log p_\theta(\y|\z)]} \\
I(\x ; \z) &= \E_{\x \sim p(\x)}[\KL{p_\theta(\z|\x)}{p_\theta(\z)}],
\end{align*}
where ${\rm KL}$ denotes the Kullback-Leibler divergence.
We can therefore approximate the IB Lagrangian empirically as 
\begin{equation}
\label{eq:empirical-IB}
\L = \frac{1}{N} \sum_{i=1}^N
\E_{z \sim p(\z|\x_i)}[-\log p(\y_i|\z)] + \beta \KL{p_\theta(\z|\x_i)}{p_\theta(\z)}.
\end{equation}
Notice that the first term simply is the average cross-entropy, which is the most commonly used loss function in deep learning. The second term can then be seen as a regularization term.
In fact, many classical regularizers, like the $L_2$ penalty, can be expressed in the form of \cref{eq:empirical-IB} (see also \cite{gal2015bayesian}).
In this work, we interpret the KL term as a reuglarizer that penalizes the transfer of information from $\x$ to $\z$. In the next section, we discuss ways to control such information transfer through the injection of noise.

\begin{rmk*}[Deterministic vs. stochastic representations]    
    Aside from being easier to work with, stochastic representations
   can attain a lower value of the IB Lagrangian than any
    deterministic representation. For example, consider
    the task of reconstructing single random bit $y$ given a noisy observation $x$.
    The only deterministic representations are equivalent to
    the either the noisy observation itself or to the trivial constant map.
    It is not difficult to check that
    for opportune values of $\beta$ and of the noise, neither 
    realize the optimal tradeoff reached by a suitable stochastic representation.
\end{rmk*}

\begin{rmk*}[Approximate sufficiency]
The quantity $I(x;y|z) = H(y|z) - H(y|x) \geq 0$ can be seen as a measure of the distance between $p(x,y,z)$ and the closest distribution $q(x,y,z)$ such that $\x\to\z\to\y$ is a Markov chain. Therefore, by minimizing \cref{eq:IB-lagrangian} we find representations that are increasingly ``more sufficient'', meaning that they are closer to an actual Markov chain.
\end{rmk*}

\section{Disentanglement}
\label{sec:disentanglement}

In addition to sufficiency and minimality, ``disentanglement of hidden factors'' is often cited as a desirable property of a representation, but seldom formalized. We can quantify disentanglement by measuring the \emph{total correlation},
or \emph{multivariate mutual information}, defined as
\begin{align*}
\TC(\z) &:= \sum_j H(z_j) - H(\z) \\
&= \KL{q(\z)}{\textstyle \prod_j q_j(z_j)}.
\end{align*}
Notice that the components of $\z$ are mutually independent if and only if $\TC(z)$ is zero. Adding this as a penalty in the IB Lagrangian, with a factor $\gamma$ yields
\begin{multline*}
\L = \frac{1}{N} \sum_{i=1}^N
\E_{z \sim p(\z|\x_i)}[-\log p(\y_i|\z)] + \\ + \beta \KL{p_\theta(\z|\x_i)}{p_\theta(\z)} + \gamma \TC(\z).
\end{multline*}
In general, minimizing this augmented loss is intractable, since to compute both the KL term and the total correlation, we need to know the marginal distribution $p_\theta(\z)$, which is not easily computable. However, the following proposition, that we prove in \Cref{sec:prior}, shows that if we choose $\gamma = \beta$, then the problem simplifies, and can be easily solved by adding an auxiliary variable.
\begin{prop}
The minimization problem
\begin{multline*}
\operatorname{minimize}_p \frac{1}{N} \sum_{i=1}^N
\E_{z \sim p(\z|\x_i)}[-\log p(\y_i|\z)] + \\ + \beta \set{\KL{p(\z|\x_i)}{p(\z)} + \TC(\z)},
\end{multline*}
is equivalent to the following minimization in two variables
\begin{multline*}
\operatorname{minimize}_{p,q} \frac{1}{N} \sum_{i=1}^N
\E_{z \sim p(\z|\x_i)}[-\log p(\y_i|\z)] + {}\\{} + \beta \KL{p(\z|\x_i)}{{\textstyle \prod_{i=1}^{|z|} q_i(\z_i)}}.
\end{multline*}
\end{prop}
In other words, minimizing the standard IB Lagrangian assuming that the activations are independent, i.e.\ having $q(\z) = \prod_i q_i(\z_i)$, is equivalent to enforcing disentanglement of the hidden factors. It is interesting to note that this independence assumption is already adopted often by practitioners on grounds of simplicity, since the actual marginal $p(\z) = \int_\x p(\x,\z)d\x$ is often incomputable. That using a factorized model results in ``disentanglement'' was also observed empirically by \cite{higgins17beta} which, however, introduced an ad-hoc metric based on classifiers of low VC-dimension, rather than the more natural Total Correlation adopted here.

In view of the previous proposition, from now on we will assume that the activations are independent and ignore the total correlation term.

\section{Information Dropout}
\label{sect-information}

Guided by the analysis in the previous sections, and to emphasize the role of stochasticity, we consider representations $\z$ obtained by computing a deterministic map $f(\x)$ of the data (for instance a sequence of convolutional and/or fully-connected layers of a neural network), and then multiplying the result component-wise by a random sample $\epsilon$ drawn from a parametric noise distribution $p_\alpha$ with unit mean and variance that depends on the input $\x$:
\begin{align*}
\varepsilon &\sim p_{\alpha(\x)}(\varepsilon), \\
\z &= \varepsilon \odot f(\x),
\end{align*}
where ``$\odot$'' denotes the element-wise product.
Notice that, if $p_{\alpha(\x)}(\varepsilon)$ is a Bernoulli distribution rescaled
to have mean $1$, this reduces exactly to the classic binary 
dropout layer. As we discussed in \Cref{sec:previous-work},
there are also variants of dropout that use different distributions.

\begin{center}
\begin{tikzpicture}[x=1cm, y=0.7cm]
    \node[obs] (x) {$x$};
    \node[latent, right = of x] (fx) {$f(x)$};
    
    \node[ right = of fx] (z) {$\z=\varepsilon \odot f(\x)$};
    \node[latent, below = of z] (e) {$\varepsilon$};

    \node[latent, right = of z] (y) {$y$};
    
    \edge[] {x}{fx};
    \edge[] {fx}{e};
    \edge[] {fx}{z};
    \edge[] {e}{z};
    \edge[] {z}{y};
\end{tikzpicture}
\end{center}

A natural choice for the 
distribution $p_{\alpha(\x)}(\varepsilon)$, which also simplifies the theoretical analysis,
is the log-normal distribution $p_{\alpha(\x)}(\varepsilon) = \log \N(0,\alpha_\theta^2(\x))$.
Once we fix this noise distribution, given the above expression for $\z$,
we can easily compute the distribution $p_\theta(\z|\x)$ that appears in \cref{eq:empirical-IB}.
However, to be able to compute the KL-divergence term, we still need to fix a prior distribution $q_\theta(\z)$.
The choice of this prior largely depends on the expected distribution of the activations $f(\x)$.
Recall that, by \Cref{sec:disentanglement}, we can assume that all activations are independent, thus simplifying the computation.
Now, we concentrate on two of the most common activation functions,
the \emph{rectified linear unit} (ReLU), which is easy to compute and works well in practice,
and the \emph{Softplus} function, which can be seen as a strictly positive and differentiable approximation of ReLU.

A network implemented using only ReLU and a final Softmax layer has the
remarkable property of being scale-invariant, meaning that multiplying all weights, biases,
and activations by a constant does not change the final result. Therefore, from a theoretical point of view,
it would be desirable to use a scale-invariant prior. The only such prior is the improper log-uniform,
$q(\log(z)) = c$, or equivalently $q(z) = c/z$, which was also suggested by \cite{kingma2015variational}, but as a prior for the weights of the network, rather than the activations.
Since the ReLU activations are frequently zero,
we also assume $q(z=0) = q_0$ for some constant $0 \leq q_0 \leq 1$. Therefore, the final prior      
has the form $q(z)=q_0 \delta_0(z) + c/z$, where $\delta_0$ is the Dirac delta in zero.
In \Cref{subfig:relu}, we compare this prior distribution with the actual
empirical distribution $p(z)$ of a network with ReLU activations.

In a network implemented using Softplus activations, a log-normal is a good fit of the distribution of the activations.
This is to be expected, especially when using batch-normalization, since the pre-activations will approximately follow a normal distribution with zero mean,
and the Softplus approximately resembles a scaled exponential near zero. Therefore, in this case we suggest using a log-normal distribution as our prior $q(z)$. In \Cref{subfig:softplus}, we compare this prior
with the empirical distribution $p(z)$ of a network with Softplus activations.

\begin{figure}

\begin{subfigure}{.45\linewidth}
    \centering
    \includegraphics[width=0.99\linewidth]{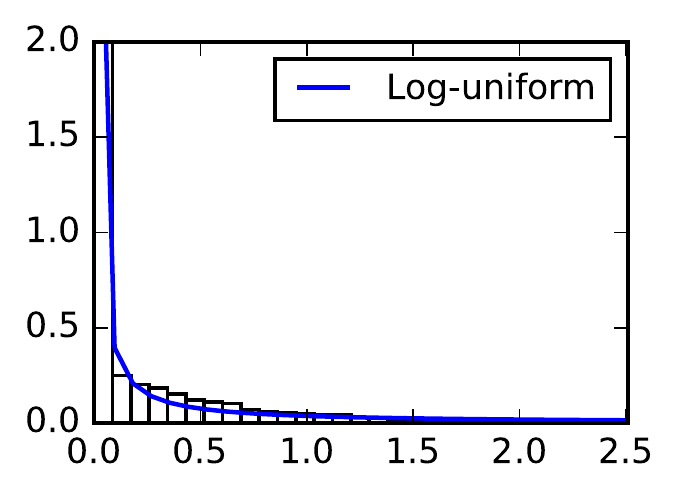}
    \caption{\label{subfig:relu} Histogram of ReLU activations}
\end{subfigure}
\begin{subfigure}{.45\linewidth}
    \centering
    \includegraphics[width=0.99\linewidth]{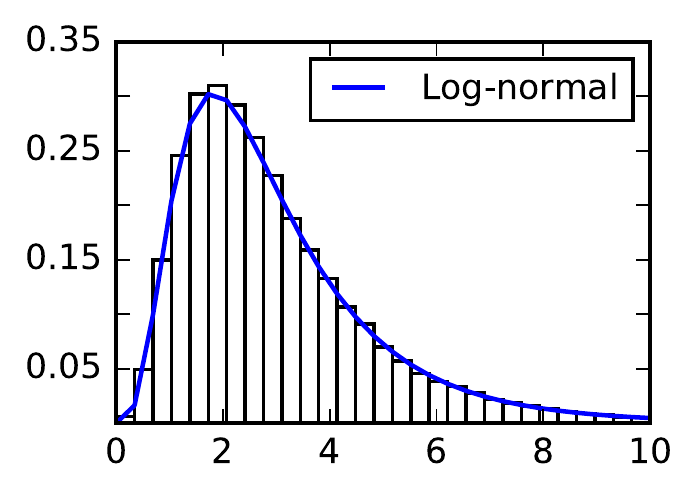}
    \caption{\label{subfig:softplus} Histogram of Softplus activations}
\end{subfigure}
\caption{
    Comparison of the empirical distribution $p(z)$ of the post-noise activations with our proposed prior
    when using: 
    (\subref{subfig:relu}) ReLU activations, for which we propose a log-uniform prior,
    and (\subref{subfig:softplus}) Softplus activations, for which we propose a log-normal prior.
    In both cases, the empirical distribution approximately follows the proposed prior.
    Both histograms where obtained from the last dropout layer of the All-CNN-32 network
    described in \Cref{table:networks}, trained on CIFAR-10.
}
\label{fig:activatios}
\end{figure}

Using these priors, we can finally compute the KL divergence term in 
\cref{eq:empirical-IB} for both ReLU activations and Softplus activations. We prove the 
following two propositions in \Cref{sec:kl-divergence}.

\begin{prop}[Information dropout cost for ReLU]
Let $z = \varepsilon \cdot f(x)$,
where $\varepsilon \sim p_\alpha(\varepsilon)$, and assume $p(z) = q\delta_0(z) + c/z$.
Then, assuming $f(x) \neq 0$, we have
\[
\KL{p_\theta(z|x)}{p(z)} = -H(p_{\alpha(x)}(\log \varepsilon)) + \log c
\]
In particular, if $p_\alpha(\varepsilon)$ is chosen to be the log-normal distribution $p_\alpha(\varepsilon) = \log \N(0,\alpha^2_\theta(x))$,
we have
\begin{equation}
\label{eq:kl-div-relu}
\KL{p_\theta(z|x)}{p(z)} = - \log \alpha_\theta(x)  + const.
\end{equation}
If instead $f(x)=0$, we have
\[
\KL{p_\theta(z|x)}{p(z)} = -\log q.
\]
\end{prop}

\begin{prop}[Information dropout cost for Softplus]
Let $z = \varepsilon \cdot f(x)$,
where $\varepsilon \sim p_\alpha(\varepsilon) = \log \N(0, \alpha_\theta^2(x))$, and assume $p_\theta(z) = \log \N(\mu,\sigma^2)$.
Then, we have
\begin{equation}
\KL{p_\theta(z|x)}{p(z)} = \frac{1}{2\sigma^2} \pa{\alpha^2(x) + \mu^2} - \log \frac{\alpha(x)}{\sigma} - \frac{1}{2}.
\end{equation}
\end{prop}

Substituting the expression for the KL divergence in \cref{eq:kl-div-relu} inside \cref{eq:empirical-IB},
and ignoring for simplicity the special case $f(x)=0$, we obtain the following loss function for ReLU activations
\begin{equation}
\L = \frac{1}{N} \sum_{i=1}^N \E_{\z\sim p_{\theta}(\z|\x_i)}[\log p(\y_i|\z)] + \beta \log \alpha_\theta(\x_i),
\end{equation}
and a similar expression for Softplus. Notice that
the first expectation can be approximated by sampling
(in the experiments we use one single sample, as customary for dropout),
and is just the average cross-entropy term that is typical in deep learning.
The second term, which is new, penalizes the network for choosing a low variance for the noise,
i.e.\ for letting more information pass through to the next layer.
This loss can be optimized easily using stochastic gradient descent and the
reparametrization trick of \cite{kingma2013auto} to back-propagate the gradient through
the sampling operation. 

\section{Variational autoencoders and Information Dropout}
\label{sec:vae}
In this section, we outline the connection between variational autoencoders \cite{kingma2013auto} and Information Dropout.
A variational autoencoder (VAE) aims to reconstruct, given a training dataset $\D=\set{\x_i}$,
a latent random variable $\z$ such that the observed data $\x$ can be thought as being
generated by the, usually simpler, variable $\z$ through some unknown generative process $p_\theta(\x|\z)$. In practice, this is done by minimizing the negative variational lower-bound to the marginal log-likelihood of the data
\begin{multline*}
\L(\theta) = \frac{1}{N} \sum_{i=1}^N \E_{\z \sim p_\theta(\z|\x_i)}[-\log p_\theta(\x_i|\z)] \\ +  \KL{p_\theta(\z|\x_i)}{p(\z)},
\end{multline*}
which can be optimized easily using the SGVB method of \cite{kingma2013auto}.
Interestingly, when the task is reconstruction, that is when $\y=\x$, the IB loss function in \cref{eq:empirical-IB}
reduces to

\begin{IEEEeqnarray*}{l}
\label{eq:information-dropout-vae}
\L(\theta) = \frac{1}{N} \sum_{i=1}^N \E_{\z \sim p_\theta(\z|\x_i)}[-\log p_\theta(\x_i|\z)] \\
\IEEEeqnarraymulticol{1}{r}{
+\> \beta \KL{p_\theta(\z|\x_i)}{p(\z)}.
}\IEEEyesnumber
\end{IEEEeqnarray*}

Therefore, by letting $\beta = 1$ in the previous expression, we obtain exactly the 
loss function of a variational autoencoder, that is, the representation $\z$
computed by the Information Dropout layer coincides with the latent
variable $\z$ computed by the VAE.
This is in part to be expected, since the objective of Information Dropout 
is to create a representation of the data that is minimal sufficient for the task of reconstruction,
and the latent variables of a VAE can be thought as such a representation.
The term $\beta$ in this case can be seen as managing the trade off between
the fidelity of the reconstruction of the input from the representation (measured by the  cross-entropy),
against the compression factor (complexity) of the representation (measured by the KL-divergence). As anticipated, Bayesian theory would prescribe $\beta = 1$, whereas it has been observed empirically that other choices can yield better performance. In the IB framework, the choice of $\beta$ is for the designer or model selection algorithm to choose.

Taking inspiration by experimental evidence in neuroscience, a contemporary work by Higgins et al. \cite{higgins17beta} also suggests the use of the loss function in eq. (\ref{eq:information-dropout-vae}) to train a VAE. They prove experimentally that for higher values
of $\beta$ the resulting representation $\z$ is increasingly disentangled. This result is 
indeed compatible with our observation in \Cref{sec:disentanglement}, and in \Cref{sect-experiments}
we prove related
experimental results in a more general situation.

\section{Experiments}
\label{sect-experiments}

The goal of our experiments is to validate the theory, by showing that indeed increasing noise level yields reduced dependency on nuisance factors, a more disentangled representation, and that by adapting the noise level to the data we can better exploit architectures of limited capacity.

To this end, we first compare Information Dropout with the Dropout baseline on several standard benchmark datasets using different
networks architecture, and highlight a few key properties.
All the models were implemented using TensorFlow \cite{tensorflow2015-whitepaper}.
As \cite{kingma2015variational} also notice, letting the variance of the noise grow excessively
 leads to poor generalization. To avoid this problem,
we constraint $\alpha(x) < 0.7$, so that the maximum variance
of the log-normal error distribution will be approximately 1, the same as binary dropout when using a drop probability of 0.5.
In all experiments we divide the KL-divergence term by the number of training samples,
so that for $\beta=1$ the scaling of the KL-divergence term
in similar to the one used by Variational Dropout (see \Cref{sec:previous-work}). 

\newcommand{\cluttered}[1]{
    \centering
    \includegraphics[width=0.95\linewidth]{images/cluttered_#1.pdf}
}
\begin{figure}
\centering
\cluttered{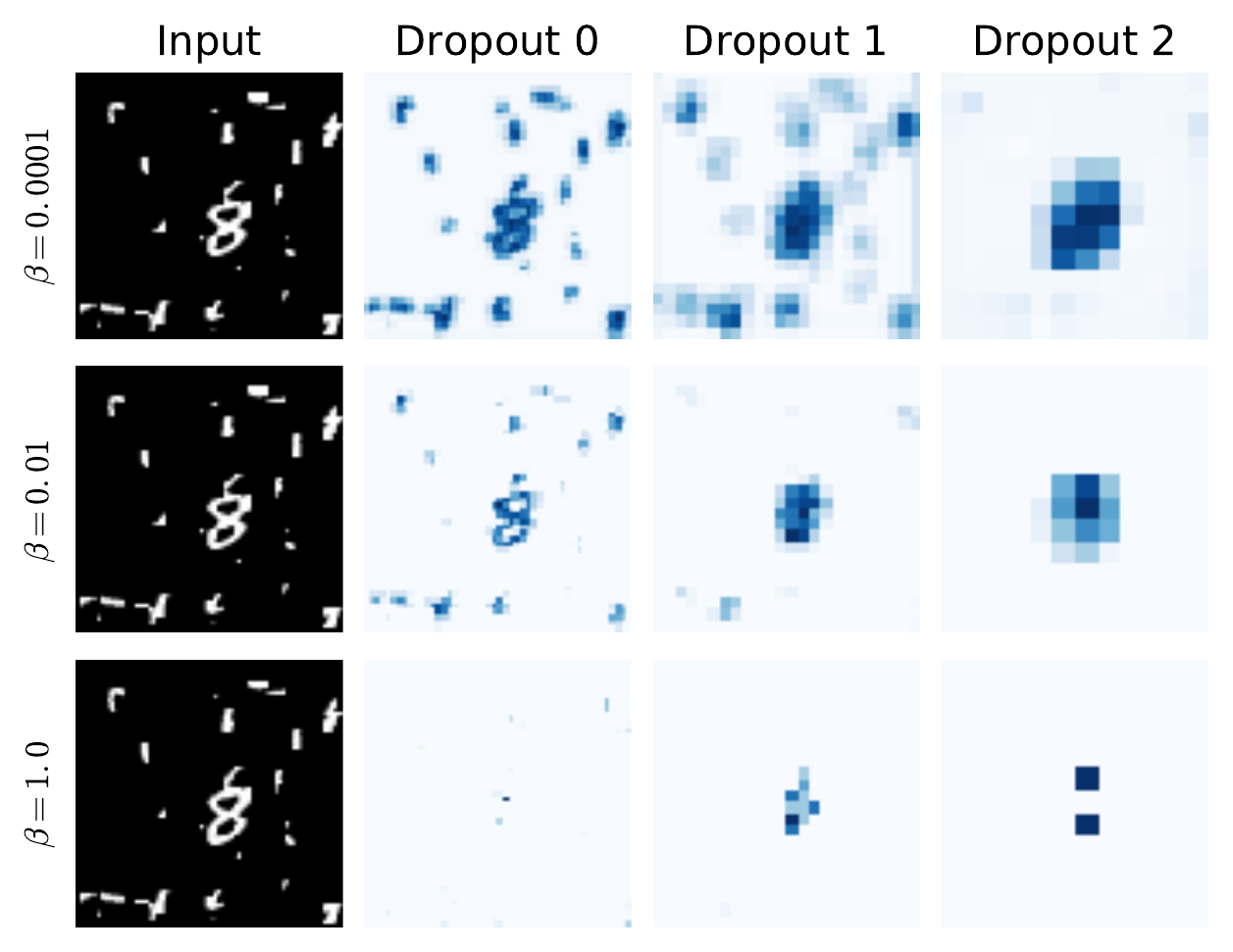}


\caption{
\label{fig:nuisances}
Plot of the total KL-divergence at each spatial location in the first three Information Dropout
layers (respectively of sizes 48x48, 24x24 and 12x12) of All-CNN-96 (see \Cref{table:networks}) trained on Cluttered MNIST with different values of $\beta$.
This measures how much information from each part of the image the Information
Dropout layer is letting flow to the next layer.
While for low value $\beta$ information about the nuisances
is still transmitted to the next layers, for higher value of $\beta$
the Information Dropout layers drop the information as soon as the receptive field 
is big enough to recognize it as a nuisance.
The resulting representation is therefore more robust
to nuisances, and provides better generalization performances.
Unlike in classical dropout or Variational Dropout,
the noise added by Information Dropout is tailored to the specific sample,
to the point that the KL-divergence alone provides enough information to localize the digit.
We provide more plots for different input samples in \Cref{sec:plots}.
}
\end{figure}

\begin{figure}

\begin{subfigure}{0.48\linewidth}
\centering
\includegraphics[width=0.99\linewidth]{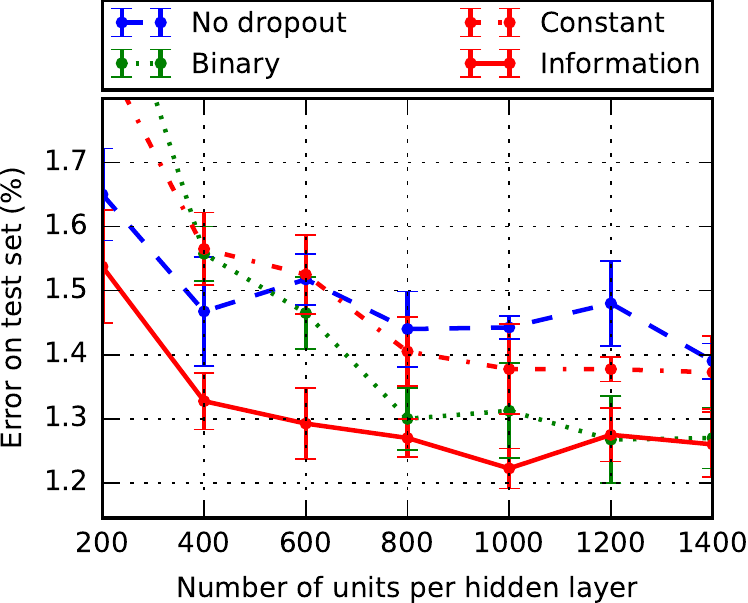}
\caption{MNIST}
\label{subfig:MNIST}
\end{subfigure}
\ 
\begin{subfigure}{0.48\linewidth}
\centering
\includegraphics[width=0.99\linewidth]{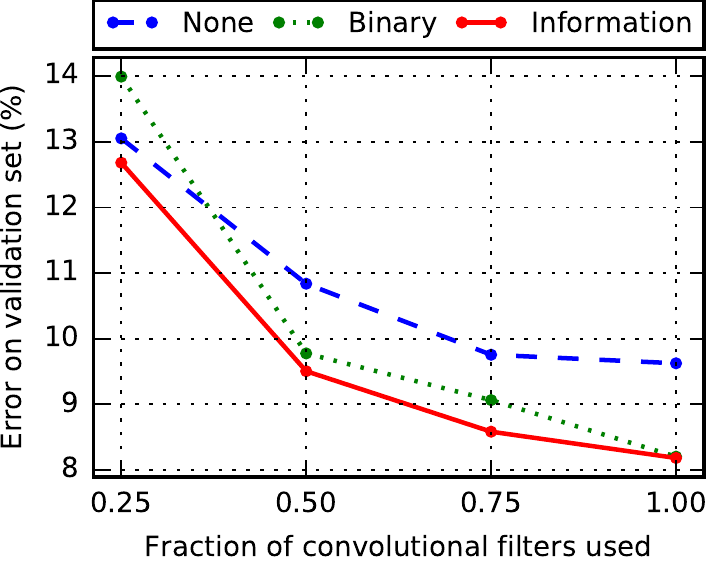}
\caption{CIFAR-10}
\label{subfig:CIFAR}
\end{subfigure}
\caption{\label{fig:MNIST}
(\subref{subfig:MNIST}) Average classification error on MNIST over 3 runs of several dropout methods applied to a fully connected network with three hidden layers and ReLU activations.
Information dropout outperforms binary dropout, especially on smaller networks, possibly because
dropout severely reduces the already limited capacity of the network, while Information Dropout
can adapt the amount of noise to the data and the size of the network.
Information dropout also outperforms a dropout layer
that uses constant log-normal noise with the same variance,
confirming the benefits of adaptive noise.
(\subref{subfig:CIFAR}) Classification error on CIFAR-10 for several dropout methods applied to the All-CNN-32 network (see \Cref{table:networks}) using Softplus activations.
}
\end{figure}

\newcommand{\occlusion}[1]{
    \includegraphics[width=1cm]{images/occluded_#1.png}
}

\begin{figure}
\centering

\begin{subfigure}{0.3\linewidth}
\centering



\occlusion{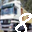}
\occlusion{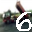}

\vspace{1em}

\occlusion{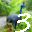}
\occlusion{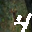}

\vspace{0.5cm}

\end{subfigure}
\begin{subfigure}{0.65\linewidth}
\centering
\includegraphics[width=0.98\linewidth]{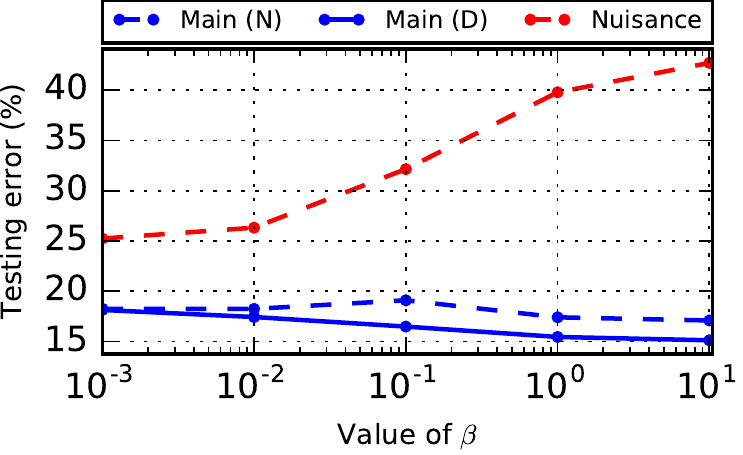}
\end{subfigure}

\caption{\label{fig:invariance}
A few samples from our Occluded CIFAR dataset and the plot of the testing error on the main task (classifying the CIFAR image)
and on the nuisance task (classifying the occluding MNIST digit) as $\beta$ varies.
For both tasks, we use
the same representation of the data trained for the main task using Information Dropout.
For larger values of $\beta$ the representation is increasingly more invariant to nuisances,
making the nuisance classification task harder, but improving the performance on the main task by preventing overfitting.
For the nuisance task, we test using the learned noisy representation of the data, since we are interested specifically in the effects of the noise. For the main task, we show the result both using the noisy representation (N), and the deterministic representation (D) obtained by disabling the noise at testing time.
}
\end{figure}

\textbf{Cluttered MNIST.} To visually asses the ability of Information Dropout to create a representation that is increasingly insensitive to nuisance factors,
we train the All-CNN-96 network (\Cref{table:networks}) for classification
on a Cluttered MNIST dataset \cite{mnih2014recurrent}, consisting of $96\times 96$ images containing a single MNIST digit together with 21 distractors.
The dataset is divided in 50,000 training images and 10,000 testing images. As shown in \Cref{fig:nuisances}, for small values of $\beta$, the network lets through both the objects of interest (digits) and distractors, to upper layers.  By increasing the value of $\beta$, we force the network to disregard the least discriminative components of the data, thereby building a better representation for the task. 
This behavior depends on the ability of Information Dropout to learn the structure of the nuisances
in the dataset which, unlike other methods, is facilitated by the ability to select noise level on a per-sample basis.

\textbf{Occluded CIFAR.} Occlusions are a fundamental phenomenon
in vision, for which it is difficult to hand-design invariant representations.
To assess that the approximate minimal sufficient representation produced by Information Dropout has this invariance property,
we created a new dataset by occluding images from CIFAR-10 with digits from MNIST (\Cref{fig:invariance}).
We train the All-CNN-32 network (\Cref{table:networks})
to classify the CIFAR image.
The information relative to the occluding MNIST digit is then a nuisance for the task, and
therefore should be excluded from the final representation.
To test this, we train
a secondary network to classify the nuisance MNIST digit using only the
the representation learned for the main task.
When training with small values of $\beta$,
the network has very little pressure
to limit the effect of nuisances in the representation, so we expect the nuisance classifier to perform
better. On the other hand, increasing the value of $\beta$ we expect its performance to degrade, since the
representation will become increasingly minimal, and therefore invariant to nuisances.
The results in \Cref{fig:invariance} confirm this intuition.

\textbf{MNIST and CIFAR-10.} Similar to \cite{kingma2015variational}, to see the effect of Information Dropout on different network sizes and architectures, we train on MNIST a network with 
3 fully connected hidden layers with a variable number of hidden units,
and we train on CIFAR-10 \cite{krizhevsky2009learning} the All-CNN-32 convolutional network
described in \Cref{table:networks},
using a variable percentage of all the filters. The fully connected network
was trained for 80 epochs, using stochastic gradient descent with momentum
with initial learning rate 0.07 and dropping
the learning rate by 0.1 at 30 and 70 epochs.
The CNN was trained for 200 epochs with initial learning rate 0.05 and dropping
the learning rate by 0.1 at 80, 120 and 160 epochs.
We show the results in \Cref{fig:MNIST}.
Information Dropout is comparable or outperforms binary dropout, especially on smaller networks.
A possible explanation is that dropout severely reduces the already limited capacity of the network,
while Information Dropout can adapt the amount of noise to the data and to the size of the network so that the relevant information can still flow to the successive layers. 
\Cref{fig:information} shows how the amount of transmitted information
also adapts to the size and hierarchical level of the layer.

\textbf{Disentangling.}
As we saw \Cref{sect-information}, in the case of Softplus activations, the logarithm of the activations approximately follow a normal distribution.
We can then approximate the total correlation
using the associated covariance matrix $\Sigma$. Precisely, we have
\[
\TC(\z) = - \log |\Sigma_0^{-1} \Sigma|
\]
where $\Sigma_0 = \operatorname{diag} \Sigma$ is the variance of the marginal distribution. In \Cref{fig:correlation} we plot for different values of $\beta$ the testing error and the total correlation of the representation learned by All-CNN-32 on CIFAR-10 when using 25\% of the filters. As predicted, when $\beta$ increases the total correlation diminishes, that is, the representation becomes disentangled, and the testing error improves, since we prevent overfitting. When $\beta$ is to large, information flow is insufficient,  and the testing error rapidly increases.

\begin{figure}
\centering
\includegraphics[width=0.8\linewidth]{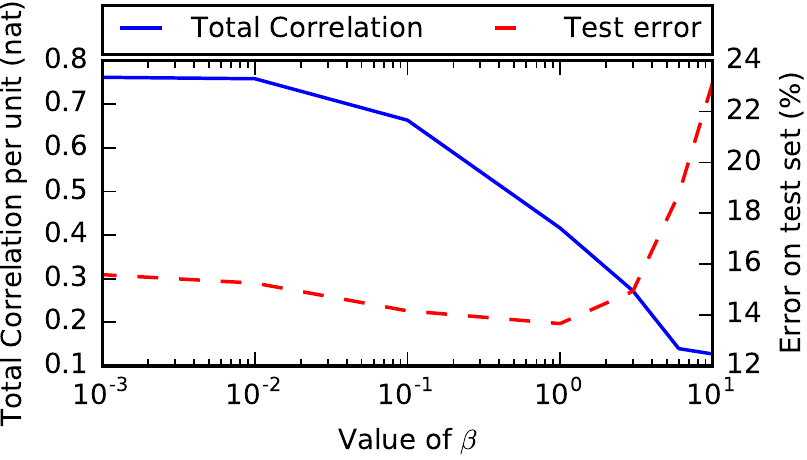}
\caption{\label{fig:correlation}
Plot of the test error and 
total correlation for different values of $\beta$ of the final layer 
of the All-CNN-32 network with Softplus activations trained on CIFAR-10 with 25\% of the filters.
Increasing $\beta$ the test error decreases (we prevent overfitting) and the representation becomes increasingly disentangled. When $\beta$ is too large, it prevents information from passing through, jeopardizing sufficiency and causingi a drastic increase in error.}
\end{figure}

\textbf{VAE.} To validate \Cref{sec:vae},
we replicate the basic variational autoencoder of \cite{kingma2013auto},
implementing it both with  Gaussian latent variables, as in the original, and with an Information Dropout layer.
We trained both implementations for 300 epochs dropping the learning rate 
by 0.1 at  30 and 120 epochs. We report the results in the following table.
The Information Dropout implementation has similar performance to the original,
confirming that a variational autoencoder can be considered
a special case of Information Dropout.

\begin{figure}[h!]

\begin{subfigure}{0.48\linewidth}
\centering
\includegraphics[width=0.98\linewidth]{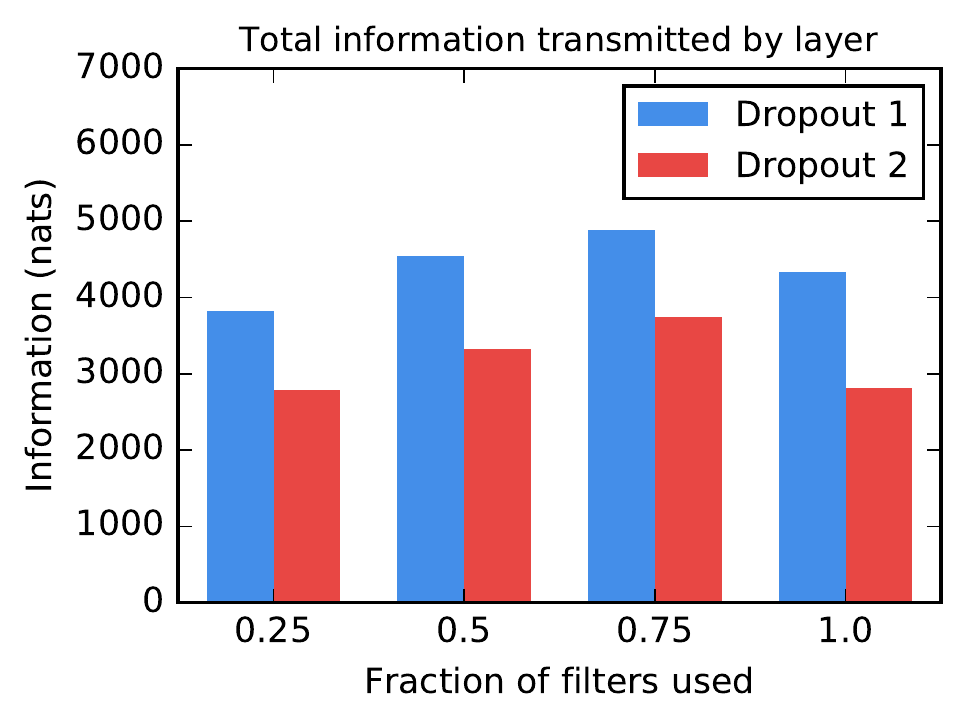}
\caption{\label{subfig:total_information}}
\end{subfigure}
\ 
\begin{subfigure}{0.48\linewidth}
\centering
\includegraphics[width=0.98\linewidth]{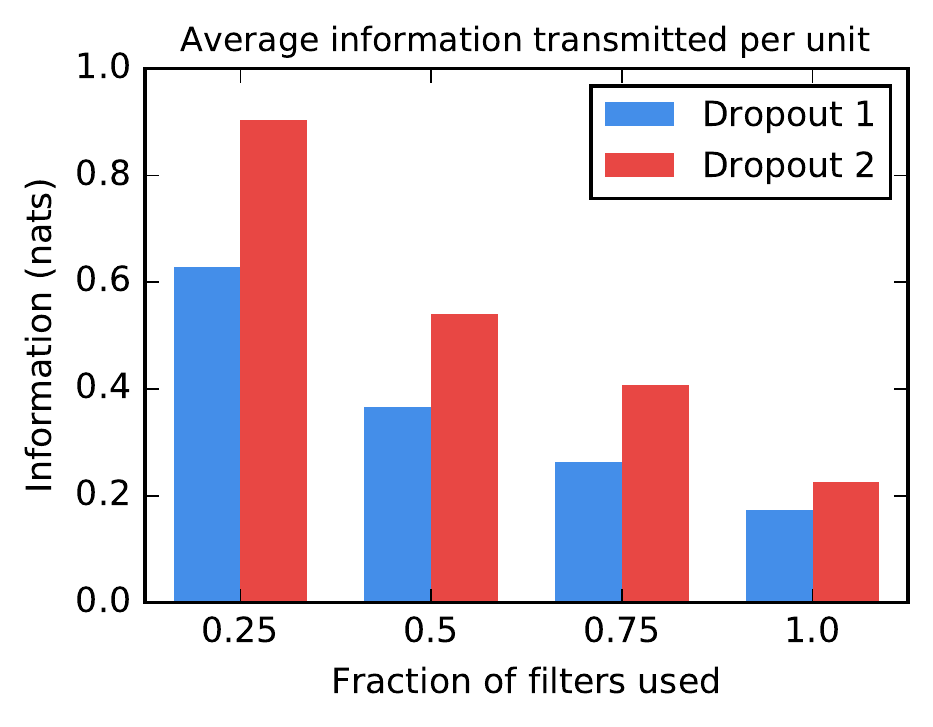}
\caption{\label{subfig:mean_information}}
\end{subfigure}
\caption{\label{fig:information}
Plots of (\subref{subfig:total_information}) the total information transmitted through the two dropout layers of a All-CNN-32 network with Softplus activations trained on CIFAR   and (\subref{subfig:mean_information}) the average quantity of information transmitted through each unit in the two layers. From (\subref{subfig:total_information}) we see that the total quantity of information transmitted does not vary much with the number of filters and that, as expected, the second layer transmits less information than the first layer, since prior to it more nuisances have been disentangled and discarded. In (\subref{subfig:mean_information}) we see that when we
decrease the number of filters, we force each single unit to let more information flow (i.e.\ we apply less noise), and that the units in the top dropout layer contain on average more  information relevant to the task than the units in the bottom dropout layer.
}
\end{figure}

\begin{table}[H]
\centering
\caption{Average variational lower-bound $\L$ on the testing dataset for a simple VAE,
where the size of the latent variable $\z$ is $256 \cdot k$ 
and the encoder/decoder each contain $512 \cdot k$ hidden units.
The latent variable $\z$ is implemented either using a Gaussian vector or using Information Dropout.
Both methods achieve a similar performance. }
\begin{center}
\begin{tabular}{ccc}
$k$ & \textbf{Gaussian} & \textbf{Information} \\
\hline \\
1 & -98.8 & -100.0 \\
2 & -99.0 & -99.1 \\
3 & -98.7 & -99.1 \\
\end{tabular}
\end{center}
\end{table}

\begin{table}
\centering
\caption{Structure of the networks used in the experiments.
The design of network is based on \cite{springenberg2014striving},
but we also add batch normalization before the activations of each layer.
Depending on the experiment,
the ReLU activations are replaced by Softplus activations, and the dropout
layer is implemented with binary dropout, Information Dropout or completely removed.}
\label{table:networks}

\small
\begin{subtable}[t]{0.48\linewidth}
\centering
\caption{All-CNN-32}
\vspace{0pt}
    \begin{tabular}{|c|}
    \hline
    Input 32x32 \\
    \hline 
    3x3 conv 96 ReLU\\
    3x3 conv 96 ReLU\\
    3x3 conv 96 ReLU stride 2 \\
    \hline
    dropout \\
    \hline
    3x3 conv 192 ReLU\\
    3x3 conv 192 ReLU\\
    3x3 conv 192 ReLU stride 2 \\
    \hline
    dropout \\
    \hline
    3x3 conv 192 ReLU\\
    1x1 conv 192 ReLU\\
    1x1 conv 10 ReLU\\
    \hline
    spatial average \\
    \hline
    softmax \\
    \hline
    \end{tabular}
\end{subtable}
\ 
\begin{subtable}[t]{0.48\linewidth}
\tiny
\centering
\caption{All-CNN-96}

\vspace{0pt}
    \begin{tabular}{|c|}
    \hline
    Input 96x96 \\
    \hline 
    3x3 conv 32 ReLU\\
    3x3 conv 32 ReLU\\
    3x3 conv 32 ReLU stride 2 \\
    \hline
    dropout \\
    \hline
    3x3 conv 64 ReLU\\
    3x3 conv 64 ReLU\\
    3x3 conv 64 ReLU stride 2 \\
    \hline
    dropout \\
    \hline
    3x3 conv 96 ReLU \\
    3x3 conv 96 ReLU \\
    3x3 conv 96 ReLU stride 2 \\
    \hline
    dropout \\
    \hline
    3x3 conv 192 ReLU\\
    3x3 conv 192 ReLU\\
    3x3 conv 192 ReLU stride 2 \\
    \hline
    dropout \\
    \hline
    3x3 conv 192 ReLU\\
    1x1 conv 192 ReLU\\
    1x1 conv 10 ReLU\\
    \hline
    spatial average \\
    \hline
    softmax \\
    \hline
    \end{tabular}
\end{subtable}

\end{table}

\section{Discussion}

We relate the Information Bottleneck principle and its associated Lagrangian to seemingly unrelated practices and concepts in deep learning, including dropout, disentanglement, variational autoencoding. For classification tasks, we show how an optimal representation can be achieved by injecting multiplicative noise in the activation functions, and therefore into the gradient computation during learning. 

A special case of noise (Bernoulli) results in dropout, which is standard practice originally motivated by ensemble averaging rather than information-theoretic considerations. Better (adaptive) noise models result better exploitation of limited capacity, leading to a method we call Information Dropout. We also establish connections with variational inference and variational autoencoding, and show that ``disentangling of the hidden causes'' can be measured by total correlation and achieved simply by enforcing independence of the components in the representation prior.

So, what may be done by necessity in some computational systems (noisy computation), turns out to be beneficial towards achieving invariance and minimality. Analogously, what has been done for convenience  (assuming a factorized prior) turns out to be the beneficial towards achieving ``disentanglement.''

Another interpretation of Information Dropout is as a way of biasing the network towards reconstructing
representations of the data that are compatible with a Markov chain
generative model, making it more suited to data coming from hierarchical models,
and in this sense is complementary to architectural constraint, such as
convolutions, that instead bias the model toward geometric tasks.

It should be noticed that injecting multiplicative noise to the activations can be thought of as a particular choice of a class of minimizers of the loss function, but can also be interpreted as a regularization terms added to the cost function, or as a particular procedure utilized to carry out the optimization. So the same operation can be interpreted as either of the three key ingredients in the optimization: the function to be minimized, the family over which to minimize, and the procedure with which to minimize. This highlight the intimate interplay between the choice of models and algorithms in deep learning.

\subsubsection*{Acknowledgments}

Work supported by  ARO, ONR, AFOSR.

\bibliographystyle{IEEEtran}
\bibliography{bibliography}

\clearpage

\appendices

\section{Computations}
\label{sec:kl-divergence}
\begin{prop*}[Information dropout cost for ReLU activations]
Let $z = \varepsilon \cdot f(x)$,
where $\varepsilon \sim p_\alpha(\varepsilon)$, and assume $p(z) = q\delta_0(z) + c/z$.
Then, assuming $f(x) \neq 0$, we have
\[
\KL{p_\theta(z|x)}{p(z)} = -H(p_{\alpha(x)}(\log \varepsilon)) + \log(c)
\]
In particular, if $p_\alpha(\varepsilon)$ is chosen to be the log-normal distribution $p_\alpha(\varepsilon) = \log \N(0,\alpha^2_\theta(x))$,
we have
\[
\KL{p_\theta(z|x)}{p(z)} = - \log \alpha_\theta(x)  + const.
\]
If instead $f(x)=0$, we have
\[
\KL{p_\theta(z|x)}{p(z)} = -\log q.
\]
\end{prop*}
\begin{proof}
If $f(x)\neq0$, then we also have $z\neq0$.
Since the KL-divergence is invariant under parameter transformations we can write
\begin{align*}
\KL{p_\theta(z|x)}{p_\theta(z)}
& = \KL{p_\theta(\log z|x)}{ p_\theta(\log z)} \\
&= \int \log \pa{\frac{p_\theta(\log z|x)}{p_\theta(\log z)}} p_\theta(\log z|x) dz \\
&= \int \log \pa{p_{\alpha(x)}(\log \varepsilon)} p_{\alpha(x)}(\log \varepsilon ) d\varepsilon \\ 
& \quad \quad \quad - \log c\\
&= - H(p_{\alpha(x)}(\log \varepsilon)) - \log c.
\end{align*}
For the second part, notice that by definition $p_{\alpha(x)} = \N(0,\alpha_\theta^2(x))$ and
\[
H( \N(0,\alpha) ) = \log \alpha_\theta(x) + \frac{1}{2} \log (2\pi e).
\]
Finally, if $f(x)=0$, then also $z=0$, so $p(z|x)=\delta_0(z)$. It is then easy to see that 
\[\KL{p_\theta(z|x)}{p(z)} = -\log p(z=0) = - \log q.\]
\end{proof}

\begin{prop*}[Information dropout cost for Softplus activations]
Let $z = \varepsilon \cdot f(x)$,
where $\varepsilon \sim p_\alpha(\varepsilon) = \log \N(0, \alpha_\theta^2(x))$, and assume $p_\theta(z) = \log \N(\mu,\sigma^2)$.
Then, we have
\[
\KL{p_\theta(z|x)}{p(z)} = \frac{1}{2\sigma^2} \pa{\alpha^2(x) + \mu^2} - \log \frac{\alpha(x)}{\sigma} - \frac{1}{2}.
\]
\end{prop*}
\begin{proof}
Since the KL divergence is invariant for reparametrizations, the divergence between
two log-normal distributions is equal to the divergence between the corresponding
normal distributions.
Therefore, using the known formula for the KL divergence of normals, we get the desired result.
\end{proof}

\section{Disentanglement}
\label{sec:prior}

In this appendix, we show that the minimization problem
\begin{multline*}
\min_p \frac{1}{N} \sum_{i=1}^N
\E_{z \sim p(\z|\x_i)}[-\log p(\y_i|\z)] + \\ + \beta \set{\KL{p(\z|\x_i)}{p(\z)} + \TC(\z)},
\end{multline*}
which is difficult in general since we do not have access to the 
joint distribution $p(\z)$, is equivalent to the following 
simpler optimization problem in two variables
\begin{multline*}
\min_{p,q} \frac{1}{N} \sum_{i=1}^N
\E_{z \sim p(\z|\x_i)}[-\log p(\y_i|\z)] + \\ + \beta \KL{p(\z|\x_i)}{{\textstyle \prod_{i=1}^{|z|} q_i(\z_i)}}.
\end{multline*}
In the following proposition, for simplicity,  we concentrate on discrete random
variables.

\begin{prop}
Let $\z=(z_1,\ldots,z_n)$ be a discrete random variable, let $p(\z|\x)$
be a generic probability distribution, and let $q(\z)=\prod_{i=1^n} q_i(z_i)$ be
a factorized prior distribution. Then, for any function $F(p)$,
 a minimization problem in the form
\[
\mathrm{minimize}_{p,q}\quad F(p) + \beta \E_{x}[\KL{p(\z|\x)}{q(\z)}],
\]
is equivalent to
\[
\mathrm{minimize}_{p}\quad F(p) + \beta \set{I_p(\z;\x) + \TC_p(\z)},
\]
where $I_p(\z;\x)$ is the mutual information and
$\TC_p(\z)$ is the total correlation of $\z$, assuming $\z \sim p(\z)$.
\end{prop}
\begin{proof}
To prove the proposition, we just need to minimize with respect to $q$ and substitute back the solution.
Adding a Lagrange multiplier for the constrain $\sum_{z_i} q_i(z_i) = 1$,
the problem can be rewritten as
\begin{align*}
\L(p,q) &= F(p) + \beta \E_{x}\bra{\sum_{\z} p(\z|\x) \log \frac{p(\z|\x)}{q(\z)} d \z } \\
&\phantom{=} {} + \lambda \pa{\sum_{z_i} q_i(z_i) - 1}\\
&= F(p) + \beta \E_{x}\bra{\sum_{\z} p(\z|\x) \log \frac{p(\z|\x)}{\prod_{j=1}^n q_j(z_j)} d \z } \\
&\phantom{=} {} + \lambda \pa{\sum_{z_i} q_i(z_i) - 1}.
\end{align*}
Taking the derivative with respect to to $p_i (\bar z_i)$ we have
\begin{align*}
\der{\L(p,q)}{p_i(\bar z_i)}
&= \beta \E_{x}\bra{\sum_{\z} q(\z|\x) \der{}{p_i(\bar z_i)} \log \frac{q(\z|\x)}{\prod_{j=1}^n p_j(z_j)}} + \lambda\\
&= - \beta \sum_{z_i = \bar z_i} \frac{\E_{x}[q(\z|\x)]}{p_i(\bar z_i)} + \lambda\\
&= - \beta \frac{q(\bar z_i)}{p(\bar z_i)} + \lambda.
\end{align*}
Setting it to zero, we obtain $p(z_i) = q(z_i)$, that is, the
optimal factorized prior is the product of the marginals.
Substituting it back in the second term (the only containing $p$), we obtain
\begin{eqnarray*}
\lefteqn{\E_{x}[\KL{q(\z|\x)}{p(\z)}] = } \\
& & = \E_{x}\bra{\sum_{\z} q(\z|\x) \log \frac{q(\z|\x)}{\prod_{j=1}^n q_j(z_j)}  } \\
& & = \E_{x}\bra{\sum_{\z} q(\z|\x) \pa{\log \frac{q(\z|\x)}{q(\z)} + \log \frac{q(\z)}{\prod_{j=1}^n q_j(z_j)}} } \\
& & = \E_{x}\bra{\sum_{\z} q(\z|\x) \log \frac{q(\z|\x)}{q(\z)}} \\
&&\phantom{=} {} + \sum_{\z} \E_{\x} [q(\z|\x)] \log \frac{q(\z)}{\prod_{j=1}^n q_j(z_j)} \\
& & = I_q(\z;\x) + \KL{q(\z)}{\textstyle \prod_j q_j(z_j)} \\
& & = I_q(\z;\x) + \TC_q(\z).
\end{eqnarray*}
\end{proof}

\section{Additional plots}
\label{sec:plots}

\begin{figure}[h]
\centering
\includegraphics[width=0.88\linewidth]{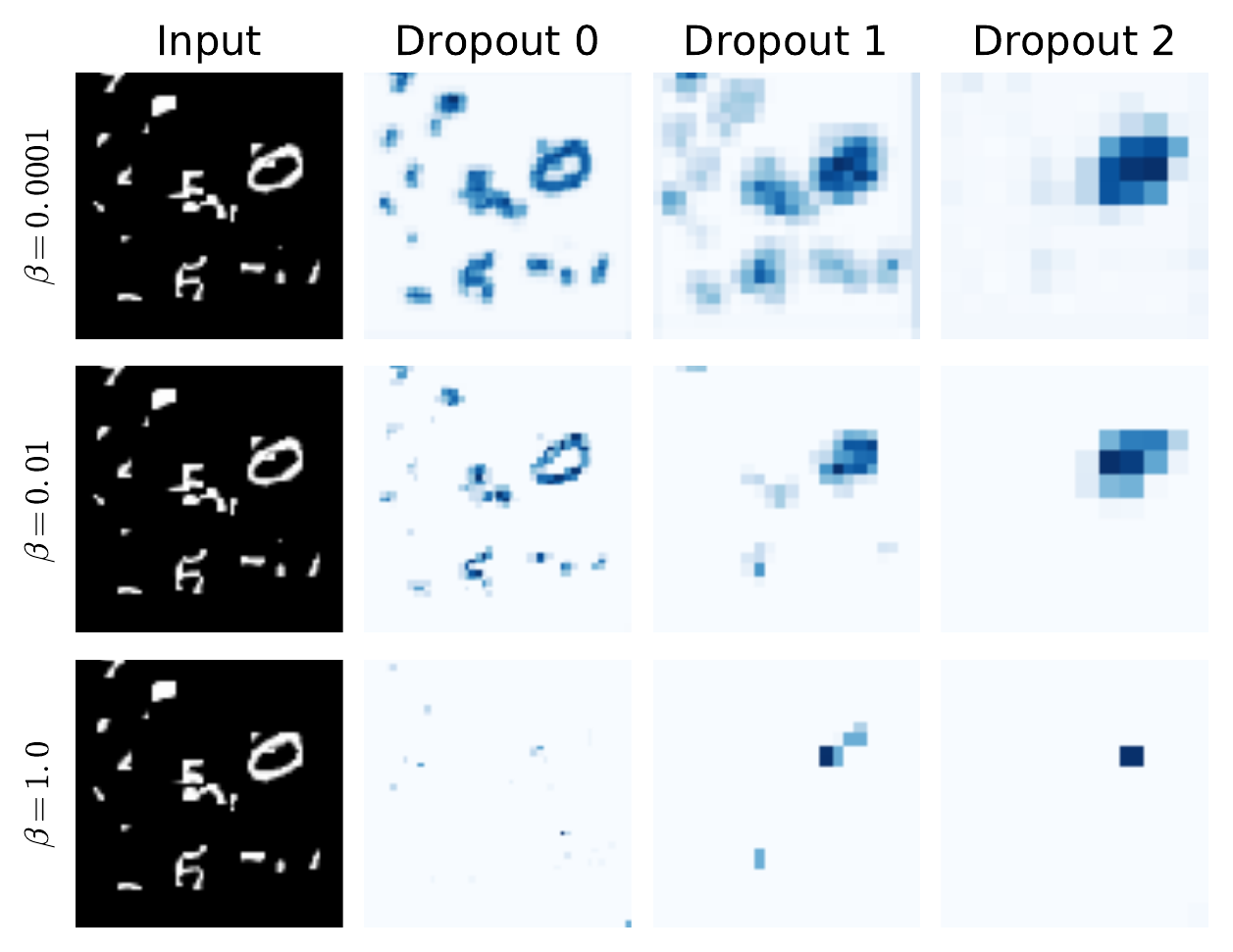}

\includegraphics[width=0.88\linewidth]{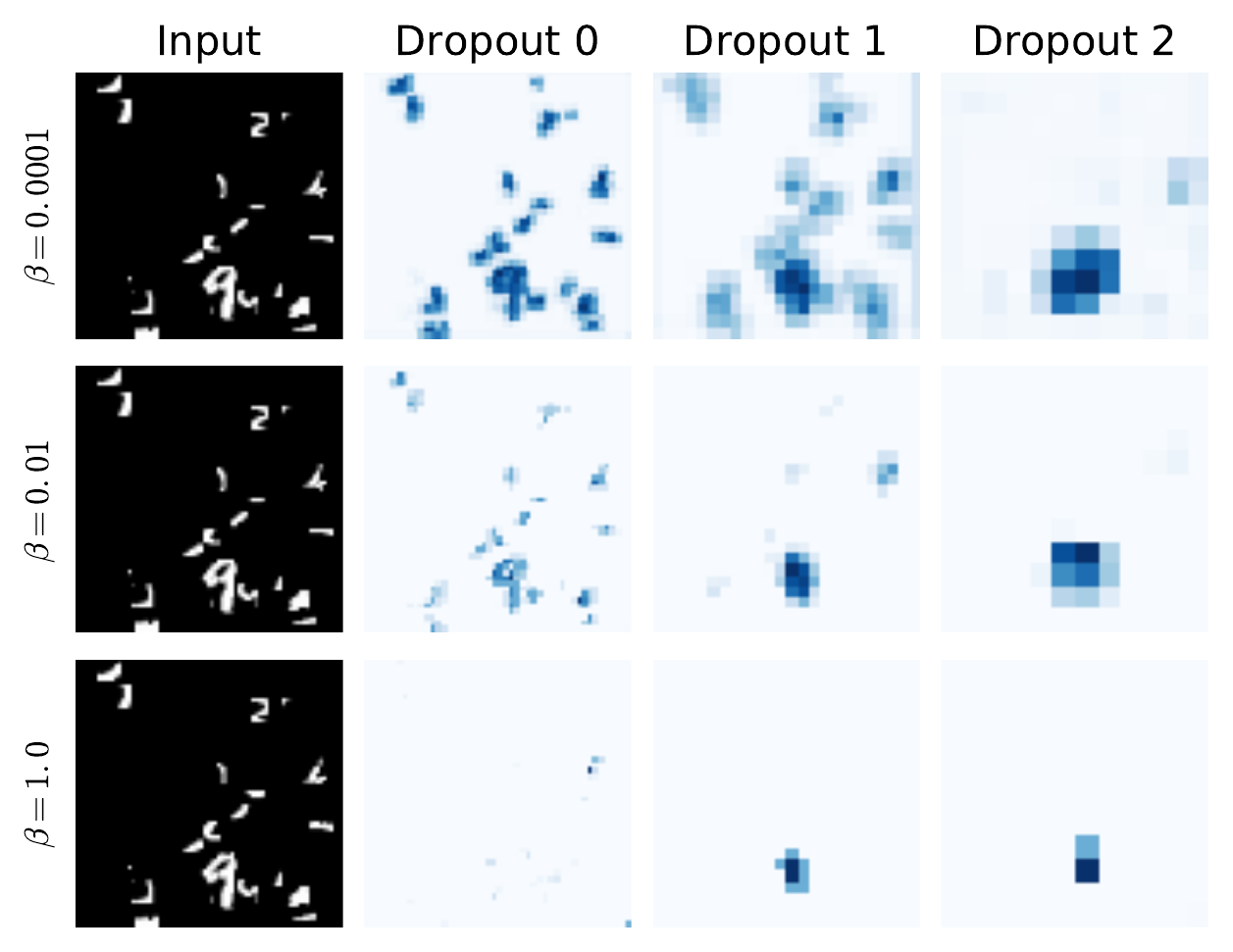}


\caption{For two more input samples, plot of the total KL divergence at each spatial location for the first three dropout layers. See \Cref{sect-experiments} and \Cref{fig:nuisances} for  detailed description. }
\end{figure}

\end{document}